\documentclass{article}
\usepackage{ijcai25}

\usepackage{times}
\usepackage{soul}
\usepackage{url}
\usepackage[hidelinks]{hyperref}
\usepackage[utf8]{inputenc}
\usepackage[small]{caption}
\usepackage{graphicx}
\usepackage{amsmath}
\usepackage{amsthm}
\usepackage{booktabs}
\usepackage{algorithm}
\usepackage{algorithmic}
\usepackage[switch]{lineno}

\usepackage{latexsym}

\usepackage{amssymb}

\usepackage{epic}
\usepackage{graphicx}
\usepackage{color}

\newtheorem{prop}{Proposition}
\newtheorem{proposition}[prop]{Proposition}
\newtheorem{defn}{Definition}
\newtheorem{definition}[defn]{Definition}
\newtheorem{cor}{Corollary}
\newtheorem{corollary}[cor]{Corollary}
\newtheorem{exmp}{Example}
\newtheorem{example}[exmp]{Example}
\newtheorem{lem}{Lemma}
\newtheorem{lemma}[lem]{Lemma}

\newtheorem{thm}{Theorem}
\newtheorem{theorem}[thm]{Theorem}

\newcommand{\mbb}[1]{\ensuremath\mathbb{#1}}

\title{Interpretable DNFs}
\author{Martin C. Cooper$^1$ \and Imane Bousdira$^2$\and Cl\'ement Carbonnel$^3$\\
\affiliations
$^1$IRIT, University of Toulouse, France\\
$^2$IRIT, INP Toulouse, France\\
$^3$LIRMM, CNRS, University of Montpellier, France\\
\emails \{cooper, imane.bousdira\}@irit.fr, clement.carbonnel@lirmm.fr
}

\begin{document}

\maketitle

\begin{abstract}
A classifier is considered interpretable if each of its decisions has an
explanation which is small enough to be easily understood by a human user.
A DNF formula can be seen as a binary classifier $\kappa$ over boolean domains.
The size of an explanation of a positive decision taken
by a DNF $\kappa$ is bounded by the size of the terms in $\kappa$,
since we can explain a positive decision by 
giving a term of $\kappa$ that evaluates to true.
Since both positive and negative decisions must be explained, 
we consider that \emph{interpretable} DNFs
are those $\kappa$ for which both $\kappa$ and
$\overline{\kappa}$ can be expressed
as DNFs composed of terms of bounded size. In this paper,
we study the family of $k$-DNFs
whose complements can also be expressed as $k$-DNFs.
We compare two such families, namely depth-$k$ decision trees
and nested $k$-DNFs, a novel family of models.
Experiments indicate that nested $k$-DNFs are
an interesting alternative to decision trees
in terms of interpretability and accuracy.
\end{abstract}

\section{Introduction}

Interpretable models are critical in machine learning applications
requiring accountability of decisions~\cite{Rudin19,Molnar}. In particular, there is a growing interest in models whose decisions can always be explained in a way that is comprehensible by a human user. In recent work on formal explainability~\cite{ShihCD18,IgnatievNM19,BarceloM0S20,AudemardBBKLM21,Joao24}, two notions of explanation of decisions have emerged. An abductive explanation corresponds to a minimal set of features that caused the decision, whereas a contrastive explanation corresponds to a means of changing the decision with changes to a minimal set of features. A theoretical line of research, starting from a list of desirable properties rather than a particular definition, has identified abductive explanations as the basis for determining what constitutes a sufficient reason for a decision~\cite{AmgoudB22,ecai/CooperA23}. 

In this paper, we deem a model to be interpretable if each of its decision has both a short abductive explanation and a short contrastive explanation. Observe that we are considering interpretability as an orthogonal 
question to \emph{explainability}, which depends
on whether we can efficiently find an explanation of each decision.
There is a considerable literature on the question of 
which families of models are explainable, whether explainability 
means the existence of polynomial-time or efficient-in-practice algorithms to find explanations~\cite{neurips20,icml21,HuangIICA022,jair/IzzaIM22,AIJ23,kr23,ijcai/Izza021,sat/IgnatievS21,aaai/IgnatievIS022}. 

This criterion for interpretability is very restrictive. The only commonly used family of models that are interpretable in this sense are decision trees whose depth is bounded by a small constant. In contrast, linear classifiers, random forests, decision lists and neural networks may all require a linear number of features in an explanation. However, it is theoretically possible that  very different families of interpretable models exist. The purpose of this paper is to study the structure of interpretable models in order to find a competitive alternative to decision trees.


We restrict our attention to classifiers which are functions of boolean features only. (However, most of our results can be extended to non-boolean features through binarisation.) In Section~\ref{sec:prel}, we observe that a boolean classifier $\kappa$ is interpretable if and only if both $\kappa$ and its complement $\overline{\kappa}$ are expressible as $k$-DNF formulas, where $k$ is the upper bound on the size of explanations. In Section~\ref{sec:short}, we show that such classifiers can always be expressed by short $k$-DNF formulas composed of at most $k^k$ terms. For small enough $k$, this shows that direct representation of interpretable classifiers as DNF formulas is always possible. Then, we describe in Section~\ref{sec:induced} a simple graph-based condition which guarantees that the complement of a $k$-DNF formula is also expressible in $k$-DNF and use this property to define nested $k$-DNFs, a new family of interpretable classifiers that is orthogonal to decision trees. We study the expressivity of nested $k$-DNFs in Section~\ref{sec:expressivity}. Finally, we present in Section~\ref{sec:expes} a practical algorithm for learning nested $k$-DNFs, and show empirically that classifiers constructed this way are competitive with decision trees on various datasets.


\section{Preliminaries}
\label{sec:prel}

We denote by $\mbb{F}$ the feature space, which for most of the paper will be $\{0,1\}^n$, and by ${\cal F}$ the set of features $\{1,\ldots,n\}$.

\begin{definition}
Given a function $\kappa: \mbb{F} \rightarrow \{0,1\}$ and an input
$v=(v_1,\ldots,v_n) \in \mbb{F}$, a weak abductive explanation (wAXp) of 
$(\kappa,v)$ is a subset $A$ of ${\cal F}$ such
that $\forall x = (x_1,\ldots,x_n) \in \mbb{F}$, $(\wedge_{i \in A} (x_i=v_i)) \Rightarrow \kappa(x)=\kappa(v)$.
A weak contrastive explanation (wCXp) of $(\kappa,v)$ is a subset 
$C$ of ${\cal F}$ such that $\exists x \in \mbb{F}$, 
$(\wedge_{i \in {\cal F} \setminus C} (x_i=v_i)) \land \kappa(x) \neq \kappa(v)$.
An \emph{abductive explanation (AXp)} is a subset-minimal wAXp.
A \emph{contrastive explanation (CXp)} is a subset-minimal wCXp.
\end{definition}

In order to give a formal definition of interpretability of a family
of models, we first give a parameterized definition of interpretability
of a classifier based on AXps/CXps.

\begin{definition} \label{defXp}
Let $k$ be a natural number.
A function $\kappa: \mbb{F} \rightarrow \{0,1\}$ is 
\emph{$k$-AXp-interpretable} if for each $v \in \mbb{F}$, there is an
AXp of ($\kappa,v)$ of size at most $k$. 
A non-constant function is $k$-CXp-interpretable if
for each $v \in \mbb{F}$, there is a CXp of size at most $k$.
By convention, a constant function is deemed to be $k$-CXp-interpretable.
\end{definition}

To see that $k$-AXp-interpretability and $k$-CXp-interpretability
do not coincide, consider the parity function $\kappa$ which returns 1 if the sum
of its $n$ boolean features is even and 0 otherwise. For any $v \in \mbb{F}$,
changing one feature changes the parity, which implies both that
$(\kappa,v)$ has a CXp of size $1$ and that, on the other hand, the only AXp is of size $n$.
Thus the existence of a small CXp does not guarantee the existence of
a small AXp.
On the other hand, for any $\kappa$, the existence of a small AXp 
(for all inputs) implies the existence of a small CXp, as we now show.

\begin{lemma}
A function $\kappa$ that is $k$-AXp-interpretable is also $k$-CXp-interpretable.
\end{lemma}

\begin{proof}
Suppose that $\kappa$ is $k$-AXp-interpretable.
The case of constant functions is trivial, so we assume that $\kappa$
is non-constant. Thus,
given an arbitrary input $v \in \mbb{F}$, there is 
another input $v' \in \mbb{F}$ such that $\kappa(v') \neq \kappa(v)$.
By $k$-AXp-interpretability, $(\kappa,v')$ has an AXp $A$ of size at most $k$. Let
$y_i = v_i$ if $i \in {\cal F} \setminus A$ and $y_i = v'_i$ if $i \in A$.
By definition, $\kappa(y) = \kappa(v') \neq \kappa(v)$.
Therefore, $A$ is a wCXp of $(\kappa,v)$ and hence some subset of $A$ 
is a CXp of size at most $k$.
\end{proof}

Since $k$-CXp-interpretability follows from $k$-AXp-interpretability,
this leads to a natural definition of interpretable models in terms of
$k$-AXp-interpretability.

\begin{definition}
\label{def:interpretablefamilies}
A family ${\cal M}$ of models is 
\emph{interpretable} if there is a constant $k$ such that
every classifier $\kappa \in {\cal M}$ is $k$-AXp-interpretable.
\end{definition} 

We now focus on the case where the feature space $\mbb{F}$ is boolean. Given a boolean function $\kappa$ over boolean variables $(x_1,\ldots,x_n)$, a \emph{literal} is either a variable $x_i$ or its negation $\overline{x_i}$. A boolean formula is in \emph{disjunctive normal form} (DNF) if it is a disjunction of \emph{terms}, which are conjunctions of literals. For simplicity of presentation, we freely interpret terms as either sets or conjunctions of literals depending on context. A DNF formula is in $k$-DNF if each of its terms has size at most $k$. We say that a conjunction (or set) of literals is \emph{consistent} if it does not contain both a variable and its negation. An \emph{implicant} of $\kappa$ is a consistent conjunction of literals $Q$ such that $\kappa$ maps to $1$ all assignments to $(x_1,\ldots,x_n)$ for which $Q$ evaluates to true. An implicant of $\kappa$ is \emph{prime} if it is subset-minimal. 

Given a DNF formula $D$ with variables $X$ and a consistent set of literals $Q$ over $X$, we denote by $D[Q]$ the DNF formula with variables $\{ x_i \in X : x_i \notin Q \text{ and } \overline{x_i} \notin Q \}$ obtained from $D$ by removing all the terms that contain the negation of a literal in $Q$ and replacing each remaining term $t = \bigwedge_{l \in S}l$ with $t[Q] = \bigwedge_{l \in S \backslash Q}l$. If $D_1$ and $D_2$ are DNF formulas that express respectively a boolean function and its complement, then for any choice of $Q$ the formulas $D_1[Q]$ and $D_2[Q]$ also express functions that are complements to each other. The \emph{size} of $D$, denoted by $|D|$, is the number of terms in $D$ and its \emph{length} $||D||$ is the sum of the sizes of its terms. Throughout the paper we will use $L(D)$ (resp. $T(D)$) to denote the sets of literals (resp. terms) that appear in the formula $D$.

For a boolean classifier $\kappa$, the prime implicants of $\kappa$ (resp. $\overline{\kappa}$) are in one-to-one correspondence with AXps for positive (resp. negative) decisions. The relationship between interpretability and expressibility as a $k$-DNF formula is made explicit by the following proposition.

\begin{proposition}
\label{prop:expdnf}
A binary boolean classifier $\kappa : \{0,1\}^n \rightarrow \{0,1\}$
is $k$-AXp-interpretable if and only if both $\kappa$ and
its complement are expressible as $k$-DNFs.
\end{proposition}

\begin{proof}
The `if' direction follows from the fact that a term
that evaluates to true is a wAXp (of size $k$), and hence some 
subset will be an AXp. The `only if' direction follows
from the fact that $\kappa$ (resp. $\overline{\kappa}$)
is equivalent to the disjunction of terms corresponding 
to the AXps of its positive (negative) decisions.
\end{proof}


Using Proposition~\ref{prop:expdnf}, it is straightforward to verify that a boolean function $\kappa$ is $k$-AXp-interpretable if and only if both $\kappa$ and $\overline{\kappa}$ are equivalent to the disjunction of their prime implicants of size at most $k$. The \emph{standard double-DNF expression} of a $k$-AXp-interpretable classifier is the pair $(D_\kappa, D_{\overline{\kappa}})$, where $D_\kappa$ is the DNF formula whose terms are the prime implicants of $\kappa$ of size at most $k$ and $D_{\overline{\kappa}}$ is the DNF formula whose terms are the prime implicants of $\overline{\kappa}$ of size at most $k$. 

The smallest integer $k$ such that a boolean function $\kappa$ and its complement can be expressed as $k$-DNF formulas is called the \emph{certificate complexity} of $\kappa$~\cite[Chapter 11]{complexitybook}. This measure is well studied in theoretical computer science and computational learning theory~\cite{certif1,certif2}, but little appears to be known about the structure of functions whose certificate complexity is bounded by a small constant. 
 




\begin{example}
\label{ex:dtdnf}
Decision trees are a well-known family of classifiers which have the reputation of
being interpretable. Indeed, if $k$ is the depth of a decision tree, then the corresponding
classifier $\kappa_{DT}$ and its complement $\overline{\kappa_{DT}}$ can \emph{both} be expressed
as $k$-DNFs. 
Given a path $\pi$ from the root to a leaf,
let $L(\pi)$ denote the set of literals labelling the edges in the path $\pi$. 
We assume a binary classifier, so each leaf is labelled 0 or 1. Let $P_0$ and $P_1$
denote the sets of paths from the root to, respectively, leaves labelled 0 and leaves labelled 1.
Then the classifier $\kappa_{DT}$ corresponding to the decision tree can be expressed as the following DNF:
\[ \kappa_{DT}(x) \ = \ \bigvee_{\pi \in P_1} \bigwedge_{\ell \in L(\pi)} \ell
\]
Furthermore, $\overline{\kappa_{DT}}$ can also be expressed as a DNF:
\[ \overline{\kappa_{DT}}(x) \ = \ \bigvee_{\pi \in P_0} \bigwedge_{\ell \in L(\pi)} \ell
\]
Observe that both these DNFs are $k$-DNFs since the length of 
paths is at most $k$.
\end{example}

As seen in Example~\ref{ex:dtdnf}, if $\kappa$ can be represented as a decision tree of depth $k$ then $\kappa$ is $k$-AXp-interpretable. However, the converse implication does not hold. In this paper, we are interested in identifying new families of interpretable classifiers that are orthogonal to those derived from decision trees.

\begin{example} \label{ex:2dnf}
For $k = 2$, a characterisation of $2$-DNF formulas whose complement is expressible in $2$-DNF can be derived from a recent result~\cite[Corollary 2]{kr23}. Together with Proposition~\ref{prop:expdnf}, this characterisation implies that a classifier $\kappa$ is $2$-AXp-interpretable if and only if it is equivalent to a DNF with one of the following forms (where the literals $a,b,c,d$ are arbitrary and not necessarily distinct): (i) $(a \land b) \lor (c \land d)$, (ii) $(a \land b) \lor (b \land c) \lor (c \land d)$, and (iii) $(a \land b) \lor (b \land c) \lor (c \land d) \lor (d \land a)$. Interestingly, certain DNFs of this kind cannot be represented as decision trees of depth $2$ (see Example~\ref{ex:nest} for more details). However, they all satisfy a different combinatorial criterion for $2$-AXp-interpretability that we describe in Section~\ref{sec:induced}.
\end{example}


\section{Short explanations imply few explanations}
\label{sec:short}

In this section, we show that every $k$-AXp-interpretable classifier is expressible as a $k$-DNF consisting of at most $k^k$ terms (independently of the number $n$ of features). This result gives further justification to work directly with DNF representations of $k$-AXp-interpretable classifiers when $k$ is small. In particular, this implies that if a classifier can provide an explanation of size at most $k$ for every decision, then all decisions can be explained using only $2k^k$ distinct explanations.

\begin{theorem}
\label{thm:terms}
    Every $k$-AXp-interpretable classifier is expressible as a $k$-DNF formula that contains at most $k^k$ terms.
\end{theorem}

\begin{proof}
    Let $\kappa$ be a $k$-AXp-interpretable classifier and $(D_\kappa, D_{\overline{\kappa}})$ be the standard double-DNF expression of $\kappa$. We will show that $D_\kappa$ contains at most $k^k$ terms. If $\kappa$ is constant then the theorem obviously holds, so let us assume that it is not. (This assumption implies in particular $k > 0$, $|D_\kappa| > 0$, and $|D_{\overline{\kappa}}| > 0$.) We claim that for all integers $j \geq 0$, either $|D_{\kappa}| < k^j$ or there exists a consistent set $Q$ of $j$ literals that is contained in at least $(1/k)^j \cdot |D_{\kappa}|$ terms of $D_{\kappa}$. We will prove this claim by induction on $j$.

    The base case $j=0$ is immediate because every term in $D_\kappa$ contains the empty set of literals. Now, let $j$ be such that $1 \leq j \leq k$ and suppose that the claim holds for $j-1$. If $|D_\kappa| < k^{j-1}$ then $|D_\kappa| < k^j$ and we are done. Otherwise, there exists a consistent set $Q'$ of $j-1$ literals and a set $S$ of at least $(1/k)^{j-1} \cdot |D_{\kappa}|$ terms of $D_{\kappa}$ such that every term in $S$ contains $Q'$. We distinguish two cases.
    
    Case 1: $\overline{Q'} = \{ \overline{l} : l \in Q' \}$ has non-empty intersection with every term in $D_{\overline{\kappa}}$. Then, $Q'$ is an implicant of $\kappa$. The terms of $D_\kappa$ are prime implicants of $\kappa$ and $Q'$ is contained in at least one term of $D_\kappa$, so $Q'$ is contained in exactly one term of $D_\kappa$. This implies $(1/k)^{j-1} \cdot |D_{\kappa}| \leq 1$ and hence $|D_\kappa| < k^{j}$.

    Case 2: there exists a term $t$ in $D_{\overline{\kappa}}$ whose intersection with $\overline{Q'}$ is empty. Consider the DNF formulas $D_\kappa[Q']$ and $D_{\overline{\kappa}}[Q']$. Observe that $t[Q']$ is a term of $D_{\overline{\kappa}}[Q']$, and $s[Q']$ is a term of $D_{\kappa}[Q']$ for all $s \in S$. (This last observation follows from the fact that every term in $S$ is a prime implicant of $\kappa$: these terms are consistent and contain $Q'$, so they cannot intersect $\overline{Q'}$.) If $t[Q']$ is the empty term, then $Q'$ is an implicant of $\overline{\kappa}$; this is not possible because at least one term in $D_\kappa$ contains $Q'$. In addition, as $D_\kappa[Q']$ and $D_{\overline{\kappa}}[Q']$ express functions that are complements of each other, the set $\{ \overline{l} : l \in t[Q'] \}$ must have non-empty intersection with every term in $D_{\kappa}[Q']$ and in particular with every term in $\{ s[Q'] \mid s \in S \}$. The term $t[Q']$ contains at most $k$ literals, so there exists $l \in t[Q']$ such that at least $(1/k) \cdot |S|$ terms in $S$ contain $\overline{l}$. Then, the set of literals $Q = Q' \cup \{\overline{l}\}$ is contained in at least $(1/k) \cdot |S| \geq (1/k) \cdot (1/k)^{j-1} \cdot |D_{\kappa}| = (1/k)^j \cdot |D_{\kappa}|$ terms of $D_{\kappa}$ and the claim holds by induction.

    We can now finish the proof of the theorem. Every term in $D_{\kappa}$ is a prime implicant of $\kappa$ so $D_{\kappa}$ cannot contain the same term twice. In addition, every term in $D_{\kappa}$ has size at most $k$. Then, for $j = k$ we have either $|D_\kappa| < k^k$ or $(1/k)^k \cdot |D_{\kappa}| \leq 1$ and the theorem follows.
\end{proof}

The specific bound of Theorem~\ref{thm:terms} is sharp as there exist $k$-AXp-interpretable classifiers that cannot be expressed as a DNF formula with fewer than $k^k$ terms. A concrete example is the complement of a classifier $\kappa$ corresponding to a DNF formula $D$ with $k$ terms of size exactly $k$, with all literals negative and no literal occurring twice. This function $\overline{\kappa}$ is $k$-AXp-interpretable, has $k^k$ prime implicants, and by monotonicity these implicants must be contained in distinct terms in any DNF expression of $\kappa$.

\begin{corollary}
    Let $\kappa : \{0,1\}^n \to \{0,1\}$ be a $k$-AXp-interpretable classifier over a set of features ${\cal F}$. There exists a set $E$ of at most $2k^k$ subsets of ${\cal F}$ such that for every $v \in \{0,1\}^n$, $E$ contains an AXp of size at most $k$ of $(\kappa,v)$.
\end{corollary}

\begin{proof}
    Applying Theorem~\ref{thm:terms}, we derive that $\kappa$ and $\overline{\kappa}$ can be expressed as $k$-DNF formulas of size at most $k^k$. The terms of these formulas are implicants of $\kappa$ and $\overline{\kappa}$ respectively, and we can further assume that they are prime implicants. Let $E$ be the set of all subsets of ${\cal F}$ whose features correspond exactly to a term. (Note that multiple terms may correspond to the same set of features, so $E$ can be strictly smaller than the sum of the sizes of these formulas.) Then, for any choice of $v$ at least one term evaluates to true and the corresponding set in $E$ constitutes a wAXp of size at most $k$ of $(\kappa,v)$. Finally, this term corresponds to a prime implicant (of either $\kappa$ or $\overline{\kappa}$) so no strict subset can be a wAXp.
\end{proof}

Another interesting consequence of Theorem~\ref{thm:terms} is that it provides an explicit characterisation of interpretable families of models (as per Definition~\ref{def:interpretablefamilies}).

\begin{corollary}
A family ${\cal M}$ of models is interpretable if and only if there exists a constant $k$ such that
every classifier $\kappa \in {\cal M}$ is expressible as a DNF formula of length at most $k$.
\end{corollary}

\begin{proof}
For the forward direction, if every classifier in ${\cal M}$ is $j$-AXp-interpretable then by Theorem~\ref{thm:terms} they are expressible as DNF formulas of length at most $k = j \cdot j^j$. Conversely, the complement of a DNF formula of length at most $k > 0$ is always expressible as a $k$-DNF of length at most $k^{k+1}$. Therefore, if every classifier in ${\cal M}$ is expressible as a DNF formula of length at most $k$ then ${\cal M}$ is interpretable.
\end{proof}

\section{Induced matchings and nested $k$-DNF}  \label{sec:induced}

In this section we describe a simple criterion for a classifier described by a $k$-DNF formula to be $k$-AXp-interpretable. This criterion is orthogonal to expressibility as a decision tree of depth $k$, and we will show in the subsequent section that it defines a remarkably expressive family of classifiers.

Let $D$ be a DNF formula that expresses a boolean function $\kappa$. A \emph{transversal} of $D$ is a subset of $L(D)$ that intersects every term in $D$. If we let ${\cal T}_{D}$ denote the set of all minimal transversals of $D$, then $\overline{\kappa}$ has the following canonical expression as a DNF:
\[
\overline{\kappa}(x) = \bigvee_{T \in {\cal T}_{D}} \bigwedge_{l \in T} \overline{l}
\]
Note that the canonical DNF expression of $\overline{\kappa}$ may include inconsistent terms. If $D$ does not contain two terms $t_1,t_2$ such that $t_1 \subset t_2$, then the canonical complement of the canonical complement of $D$ is $D$ itself\footnote{This is a well-known property of hypergraph dualisation, see e.g.~\cite[Chapter 2]{bergehypergraphs}.}. From this perspective, it is clear that the function expressed by a given $k$-DNF formula $D$ is $k$-AXp-interpretable if all minimal transversals of $D$ have cardinality at most $k$. 

Let $G_D=(V,E)$ be the bipartite graph with $V = L(D) \cup T(D)$ and $\{l,t\} \in E$ if and only if $l \in t$. An \emph{induced matching} of $G_D$ is a subset $M \subseteq E$ such that no two edges in $M$ share an endpoint and no edge in $E$ intersects two distinct edges in $M$. We denote by mim$(G_D)$ the maximum number of edges in an induced matching of $G_D$.

\begin{lemma}
\label{lem:mim}
Let $D$ be a $k$-DNF formula expressing a boolean function $\kappa$. If mim($G_D$) $\leq k$, then $\kappa$ is $k$-AXp-interpretable.
\end{lemma}

\begin{proof}
We show that every minimal transversal of $D$ has cardinality at most $k$. Suppose for the sake of contradiction that $D$ has a minimal transversal $T$ of size $q > k$. By minimality, for every literal $l \in T$ there exists a term $t_l \in T(D)$ such that $t_l \cap T = \{l\}$. Then, the set of edges $\{ \{l, t_l\} \mid l \in T \}$ is an induced matching of $G_D$ of size $q > k$. This is not possible because mim($G_{D}$) $\leq k$.
\end{proof}

\begin{example}
\label{ex:mim}
Consider the majority function on $2k-1$ arguments defined by
$\kappa_{\mathrm{maj}}(x_1,\ldots,x_{2k-1}) \equiv (\sum_{i=1}^{2k-1} x_i \geq k)$. This function $\kappa_{\mathrm{maj}}$ is $k$-AXp-interpretable 
since it is the disjunction of all terms 
composed of exactly $k$ positive literals and its complement 
is the disjunction of all terms composed of exactly $k$ negative literals. 

The graphs associated with these formulas do not contain induced matchings of size larger than $k$, so $\kappa_{\mathrm{maj}}$ satisfies the criterion for $k$-AXp-interpretability given by Lemma~\ref{lem:mim}. However, it is well known that any decision tree representing $\kappa_{\mathrm{maj}}$ must have depth at least $2k-1$, as any path starting from the root that alternates between positive and negative literals cannot reach a leaf before all variables have been assigned.
\end{example}

The simple condition provided by Lemma~\ref{lem:mim} already defines a new family of $k$-AXp-interpretable classifiers, those expressible by $k$-DNF formulas with no induced matchings of size $k+1$. From a practical viewpoint, the interest of this family is limited because its definition is not constructive: without a clear structure, it is difficult to design efficient heuristics for learning formulas of this kind directly from data.

We address this issue by defining a smaller family of classifiers whose structure is more explicit. Consider $k^2$ literals $\ell_{i,j}$ ($1 \leq i,j \leq k$).
We can view $\{\ell_{i,j}\}$
as a $k \times k$ matrix:
\[
{\cal L} = \begin{pmatrix}
    \ell_{1,1} & \ell_{1,2} & \ldots & \ell_{1,k} \\
     & \vdots & & \\
    \ell_{k,1} & \ell_{k,2} & \ldots & \ell_{k,k} 
\end{pmatrix}
\]
We will define a $k$-DNF $D$ composed of $m$ terms 
(where $m$ is arbitrary) whose complement is also expressible
as a $k$-DNF. For each $p=1,\ldots,m$, let $r_{pi}$ ($i=1,\ldots,k$) be $k$ integers between 0
and $k$ such that $\sum_{i=1}^{k} r_{pi} \leq k$. Then define $D$ as follows:
\[ D \ =\ \bigvee_{p=1}^{m} \ \bigwedge_{i=1}^{k} \bigwedge_{j=1}^{r_{pi}} \ell_{i,j}
\]
The condition that $\sum_{i=1}^{k} r_{pi} \leq k$ 
for each $p=1,\ldots,m$ ensures that $D$
is a $k$-DNF. We call such a DNF a \emph{nested} $k$-DNF.
The term 
\[\bigwedge_{i=1}^{k} \bigwedge_{j=1}^{r_{pi}} \ell_{i,j}
\]
of $D$ is the conjunction of, for each $i=1,\ldots,k$, the $r_{pi}$ leftmost
elements in row $i$ of the matrix $\cal L$.

\begin{proposition}
\label{prop:nest}
Every boolean function expressible as a nested $k$-DNF formula is $k$-AXp-interpretable.
\end{proposition}

\begin{proof}
Let $D =\ \bigvee_{p=1}^{m} \ \bigwedge_{i=1}^{k} \bigwedge_{j=1}^{r_{pi}} \ell_{i,j}$ be a nested $k$-DNF formula. Towards a contradiction, suppose that there exists an induced matching $M$ of size $k+1$ in $G_D$. By the pigeonhole principle, at least two literals that appear in $M$ belong to the same row $i_M$ of ${\cal L}$. Two terms are matched with these two literals, and the term with the largest value for $r_{pi_M}$ must contain both. This is impossible because $M$ is an induced matching. Applying Lemma~\ref{lem:mim}, the function expressed by $D$ is therefore $k$-AXp-interpretable.
\end{proof}

\begin{example}
\label{ex:nest}
Observe that all $k$-DNF formulas with $q$ terms are nested if $q \leq k$. Indeed, for any such formula $D$ we can set ${\cal L}$ to be a $k \times k$ matrix of literals whose $i$th row contains the literals of the $i$th term of $D$ (possibly with repetition if the term has fewer than $k$ literals). Then, for $p \leq q$ we set $r_{pi} = k$ if $p=i$, and $r_{pi} = 0$ otherwise. These parameters will produce exactly the formula $D$. In general, such formulas are not expressible as decision trees of depth smaller than $k^2$~\cite{DurdymyradovM24}.

On the other hand, the function $\kappa_{\mathrm{maj}}$ of Example~\ref{ex:mim} is not expressible as a nested $k$-DNF. We proceed again by contradiction. If $\kappa_{\mathrm{maj}}$ could be represented by a nested $k$-DNF generated from a $k \times k$ matrix ${\cal L}$, then ${\cal L}$ would contain only positive literals. Let $L_1$ be set of the literals in the first column of ${\cal L}$,
and $J$ the other positive literals. 
All terms generated from ${\cal L}$ must 
contain at least one literal from $L_1$. If $|J| \geq k$,
then there is a term consisting of $k$ positive literals not
occurring in the first column, and which therefore could not be generated.
Hence, we must have $|L_1|=k$ and $|J|=k-1$. Without loss of generality,
assume $L_1=\{x_1,\ldots,x_k\}$ and $J=\{x_{k+1},\ldots,x_{2k-1}\}$.
For each $i=1,\ldots,k$, the term $x_i \bigwedge_{x_j \in J}x_j$  
must be generated from ${\cal L}$ by taking
literals from a single row, since it contains a single literal from 
the first column of ${\cal L}$. It follows that the columns $2,3,\ldots,k$
of ${\cal L}$ contain only elements of $J$. Since $|J|=k-1$,
the second column of ${\cal L}$ must contain at least one repeated
element. Without loss of generality, assume that this repeated element is $x_{k+1}$
and that it occurs in the two rows whose first elements are $x_1$ and $x_2$.
But then, for $k \geq 3$, it is impossible to generate the term
$x_1 x_2 x_{k+2} \ldots x_{2k-1}$, since all terms containing $x_1$
and $x_2$ must also contain $x_{k+1}$. 
\end{example}

\section{Expressivity of nested $k$-DNFs}
\label{sec:expressivity}

In machine learning, it is important that the language of models
$\cal M$ used in the learning phase be sufficiently rich
to capture all functions we might wish to learn. Consider a classifier $\kappa$ that is a function of only $k$
variables $x_1,\ldots,x_k$. Both $\kappa$ and its complement $\overline{\kappa}$
can be expressed as $k$-DNFs. This is because $\kappa$ (respectively, 
$\overline{\kappa}$) is the disjunction of the terms corresponding to
the assignments to the variables $x_1,\ldots,x_k$ for which
$\kappa(x_1,\ldots,x_k)=1$ (respectively, $\overline{\kappa}(x_1,\ldots,x_k)=1$). 
All functions of $k$ variables can be expressed as depth-$k$
decision trees (with $x_i$ associated with all decision nodes
at depth $i-1$), so an obvious question is whether the same is true
for nested $k$-DNFs. We answer this question positively in the
following proposition.

\begin{proposition}
\label{prop:functions}
Every boolean function $\kappa$ of $k$ boolean variables can be
expressed as a nested $k$-DNF.
\end{proposition}

\begin{proof}
If $\kappa$ is the constant function 1, then it can
be trivially expressed as a nested $k$-DNF that contains a single term with zero literals. We can therefore assume that $\kappa$ is equal
to 0 for some assignment to the $k$ variables $x_1,\ldots,x_k$.
Without loss of generality, we assume that $\kappa(0,\ldots,0)=0$.
Let the $k \times k$ matrix of literals $\{\ell_{ij}\}$ be
\[
{\cal L} = \begin{pmatrix}
    x_1 & \overline{x_2} & \overline{x_3} & \ldots & \overline{x_k} \\
    x_2 & \overline{x_3} & \overline{x_4} & \ldots & \overline{x_1} \\
    x_3 & \overline{x_4} & \overline{x_5} & \ldots & \overline{x_2} \\
     &  & \vdots & & \\
    x_k & \overline{x_1} & \overline{x_2} & \ldots & \overline{x_{k-1}} 
\end{pmatrix}
\]
The classifier $\kappa$ can be expressed as the disjunction of terms
corresponding to assignments for which $\kappa$ is equal to 1. 
Any such term $t$ contains $h$ positive literals, where $h \geq 1$
(since a DNF satisfying $\kappa(0,\ldots,0)=0$ cannot
contain the term $\overline{x_1} \cdots \overline{x_k}$). 
Let $x_{i_j}$ ($j=1,\ldots,h$) be these positive literals.
Let $r_{i_j}=i_{j+1}-i_j$ ($j=1,\ldots,h-1$), $r_{i_h}=k+i_1-i_h$
and $r_i=0$ for all $i \notin \{i_1,\ldots,i_h\}$.
Then $t$ is the conjunction of the leftmost $r_{i_j}$ literals in row $i$
(for $i=1,\ldots,k$) of the above matrix $\cal L$. Since each term $t$
of $\kappa$ can be constructed in this way, $\kappa$ is a
nested $k$-DNF.
\end{proof}


A consequence of Proposition~\ref{prop:functions} is that nested $k$-DNF formulas can always be constructed to fit any consistent dataset provided that $k$ is large enough. In particular, the least integer $k$ such that a boolean function or its complement can be represented as a nested $k$-DNF formula is a well-defined measure that cannot exceed the number of variables (as is the case for decision trees of depth $k$).

One criterion for comparing families of models ${\cal M}$ is to estimate
the number of distinct functions that can be represented by ${\cal M}$.
Let $N_{DT}(k,n)$ and $N_{\mathit{nested}}(k,n)$ be, respectively, the number
of functions representable by a depth-$k$ decision tree or by a nested $k$-DNF,
where $n$ is the total number of variables. Recall that nested $k$-DNF formulas can be function of at most $k^2$ variables, whereas decision trees of depth $k$ may depend on (up to) $2^k-1$ variables. For this reason, it is expected that if $n$ is large enough compared to $k$ then $N_{\mathit{nested}}(k,n)$ will necessarily be smaller than $N_{DT}(k,n)$. We show that the opposite is true when $n$ is not much larger than $k$. Informally, nested $k$-DNF formulas involve fewer features than decision trees of depth $k$ but can express a greater variety of dependencies between those features.

\begin{proposition}
If $k \geq 4$ and $k^2 \leq n \leq 2^{2^{k-1}/k-1}$, then $N_{\mathit{nested}}(k,n) > N_{DT}(k,n)$.
\end{proposition}

\begin{proof} 
Every function representable by a decision tree of depth $k$ can be represented by a complete tree with $2^k$ leaves. Each of the $2^k-1$ internal nodes is associated with a variable and each of the $2^k$ leaves is associated with a class. There are $n^{2^k\!-\!1} 2^{2^k}$ such decision trees, so $N_{DT}(k,n) \leq n^{2^k\!-\!1} 2^{2^k}$.

We consider a fixed matrix 
${\cal L}$ composed of $k^2$ distinct positive literals $\ell_{i,j}$. 
By the stars and bars theorem, the number of distinct terms of the form 
$\bigwedge_{i=1}^k \bigwedge_{j=1}^{r_i} \ell_{i,j}$
where $\sum_{i=1}^k r_i = k$
is exactly $C^{2k-1}_{k-1} = 1/2 \cdot C^{2k}_{k}$. Using the inequality $C^{2k}_{k} \geq 2^{2k}/\sqrt{\pi (k + 1/3)}$, we deduce that $N_{\mathit{nested}}(k,n)$
is bounded below by $2^{2^{2k-1}/k}$ for $k \geq 4$, since each of the $1/2 \cdot C^{2k}_{k}$ terms may or may not occur in 
the nested $k$-DNF formula. It follows that $N_{\mathit{nested}}(k,n) > N_{DT}(k,n)$
if $n \leq 2^{2^{k-1}/k-1}$.
\end{proof}

Figure~\ref{fig:landscape} provides a summary of the relationship between the major classes of $k$-AXp-interpretable classifiers.

\setlength{\unitlength}{1cm}
\begin{figure}
\centering
\begin{picture}(7.6,3.8)(-1.8,0.7)
    \put(1.1, 2){\oval(5.2,2)}
    \put(2, 2.1){\oval(2.9,1.35)}
    \put(2.9, 2.2){\oval(5.2,2)}
    \put(2.9, 2.5){\oval(5.6,3.5)}
    \put(-0.6,2){\makebox(0,0){\shortstack{depth-$k$ \\ DTs}}}
    \put(2,2.1){\makebox(0,0){\shortstack{functions of \\ $k$ variables}}}
    \put(4.6,2.2){\makebox(0,0){\shortstack{$\kappa$ or $\overline{\kappa}$ is \\ a nested \\ $k$-DNF}}}
    \put(2.9,3.7){\makebox(0,0){\shortstack{$\kappa$ or $\overline{\kappa}$ is a $k$-DNFs with \\induced  matchings of size $\leq k$}}}
\end{picture}
\caption{The landscape of $k$-AXp-interpretable classifiers $\kappa$}
\label{fig:landscape}
\end{figure}
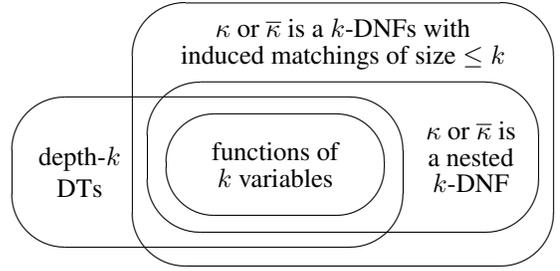

\section{Experiments}
\label{sec:expes}

In this section, we present a heuristic algorithm for finding nested $k$-DNFs, distinguished by its intuitive and straightforward design\footnote{The code is available in this GitHub \href{https://github.com/1-IM/I-DNFs}{repository}}. It is worth noting that alternative algorithms could also be considered. Next, we provide an experimental comparison with the depth-$k$ decision trees obtained by CART~\cite{DBLP:books/wa/BreimanFOS84}.

\subsection{Heuristic algorithm}

The heuristic consists of three steps: constructing the matrix, constructing the nested $k$-DNF, and a pruning phase. In Algorithm~\ref{alg:algorithm}, we show how to construct the $k \times k$ matrix $\cal L$ by proceeding row by row, where $k$ is less than or equal to the total number of features $n$. The idea is to create a matrix that will allow us, in the next step, to generate a large number of distinct and consistent terms. To achieve this, the literal $\ell_{i,j}$ ($0 \leq i,j \leq k-1$) is selected such that the $j+1$ leftmost elements in row $i$ of the matrix $\cal L$ are highly representative of class 1 while being minimally representative of class 0. A key condition is that $\ell_{i,j}$ must differ from the $j$ preceding literals in row $i$ and their negations (to avoid redundancy or inconsistency). Additionally, to encourage diversity between different rows, we exclude all literals in the first $limit=k-j$ columns (of the already-chosen rows) from the list of candidate literals for $\ell_{i,j}$, provided that at least one literal remains available for selection. The value of $limit$ is reduced accordingly if the number $2(n-j)$ of available literals is less than or equal to the number $i \times (k-j)$ of literals we would like to forbid.

Secondly, we construct the nested $k$-DNF by evaluating one term at a time, starting with terms of size $k$ and decreasing down to size 1. A term is considered for evaluation if it is consistent (i.e. it does not contain both a literal and its negation). We decide to select a term if $P \neq 0 $ and $Q < P$, where $P$ (respectively, $Q$) represents the number of examples in class 1 (respectively, class 0) that satisfy this term and are not already covered by the selected terms. Furthermore, a term is also chosen if it covers at least one example from class 1 and does not cover any example from class 0 (irrespective of whether examples have already been covered). The process stops when either all examples in class 1 are covered or there are no more terms to evaluate. Finally, we perform pruning, where we determine whether to retain each term. The same evaluation as before is applied using P and Q (i.e. we remove a term if $P = 0$ or $Q \geq P$). This time, we compare each term against all other terms, not just the previously selected terms.

\begin{algorithm}[tb]
    \caption{Construct\_matrix}
    \label{alg:algorithm}
    \textbf{Input}: k, dataset\\ 
    \textbf{Output}: matrix ${\cal L}$ 
    \begin{algorithmic}[1] 
        \FOR{ $i = 0 \,\, to \,\, k-1$}
        \FOR{$j = 0 \,\, to \,\, k-1$}
        \IF {$ i = 0 $}
        \STATE $limit = 0$
        \ELSE
        \STATE $limit = \min(k-j, \lceil (2(n-j)/i)-1 \rceil)$
        \ENDIF \\
        \textbackslash\textbackslash \, $Ec_{1} (t)$: nb. examples in class 1 that satisfy $t$ \\
        \textbackslash\textbackslash \, $Ec_{0} (t)$: nb. examples in class 0 that satisfy $t$ \\
        \STATE Calculate $ G = Ec_{1}(\ell_{i,0} ... \ell_{i,j}) \, – \, Ec_{0}(\ell_{i,0} ... \ell_{i,j}) $ for each literal not in ${\cal L}_{i,0:j} \cup \overline{{\cal L}_{i,0:j}} \cup {\cal L}_{0:i,0:limit}$ \\
	    \STATE Take as $\ell_{i,j}$ the literal that gives the greatest G
        \ENDFOR
        \ENDFOR
        \STATE \textbf{return} matrix ${\cal L}$ 
    \end{algorithmic}
\end{algorithm}

\subsection{Datasets}

A collection of datasets from the \href{https://archive.ics.uci.edu/}{UCI repository} and \href{https://www.kaggle.com/}{Kaggle} are considered, which have been used to evaluate a wide range of learning algorithms. These datasets contain various feature types, which are converted into boolean features for binary classification as in~\cite{blossom}. We employ the datasets in their original form, without any preprocessing techniques applied. Table~\ref{tab:datasets} shows, for each dataset, the number of data examples and the number of boolean features. 

\begin{table} 
    \centering
    \begin{tabular}{lrr}
        \toprule
        Dataset  & Size & Nb. boolean features \\
        \midrule
        Balance-scale  &  625 & 16   \\
        Banknote  & 1372 & 28    \\
        Car-evaluation  & 1728 & 14    \\
        Compas discretized  & 6167 & 25    \\
        Indians Diabetes  & 768 & 43   \\
        Iris  & 150 & 12   \\
        Lymph  & 148 & 68   \\
        Monks-1  & 124 & 11  \\
        Monks-2  & 169 & 11  \\
        Monks-3  & 122 & 11  \\
        Tic-tac-toe  & 958 & 27  \\
        \bottomrule
    \end{tabular}
    \caption{Description of the datasets used in the experiments.}
    \label{tab:datasets}
\end{table}

\begin{table*} [!t]
    \centering
    \begin{tabular}{l|rrr|rrr|rrr}
        \toprule
         & \multicolumn{9}{c}{Test accuracy (\%)} \\
        \cmidrule{2-10}
        Dataset & \multicolumn{3}{c}{$k=2$} & \multicolumn{3}{c}{$k=3$} & \multicolumn{3}{c}{$k=4$} \\
        \cmidrule{2-10}
        &   \hspace{0.35cm} DT   &  \hspace{0.35cm} DNF &  \hspace{0.35cm}  \(\overline{\text{DNF}}\) &   \hspace{0.35cm} DT   & \hspace{0.35cm} DNF & \hspace{0.35cm}  \(\overline{\text{DNF}}\)&  \hspace{0.35cm} DT  & \hspace{0.35cm}  DNF  &\hspace{0.35cm}  \(\overline{\text{DNF}}\)\\ 
        \midrule
        Balance-scale & \textbf{93.28} & 89.04 & \textbf{93.28} &
        \textbf{93.28} & 92.46 & \textbf{93.28} & 
        \textbf{93.28} & 92.10 & \textbf{93.28} \\
        
         Banknote    & 86.40 & \textbf{88.95} & 83.35 &
         \textbf{89.45} & 88.95 & 83.35 &
         \textbf{95.49} & 90.53 & 86.24 \\
         
         Car-evaluation   &  \textbf{85.78} & 77.80 & 73.35 & 
         86.65 & 84.51 & \textbf{89.13} &
         91.68 & 83.15 & \textbf{92.14}\\
         
         Compas discretized   &   64.47 & 64.02 & \textbf{65.71} &
         65.90 & 65.87 & \textbf{67.18} &
         66.40 & 66.07 & \textbf{67.12}  \\
         
         Indians Diabetes   & 77.01 &  \textbf{78.70} & 76.16 &
         78.18 & \textbf{79.48} & 77.48 &
         77.42 & \textbf{79.56} & 77.52 \\
         
         Iris   &  98.00 & 96.00 &  \textbf{99.33} &
         \textbf{98.00} &  97.53 & \textbf{98.00} &
         98.00 & \textbf{98.60} &  98.00 \\
         
         Lymph   &   81.33 &  76.73 & \textbf{85.33} &
        79.93 & 79.67 &  \textbf{87.13} &
        85.13 & 82.07 &  \textbf{86.07}\\
        
         Monks-1   & \textbf{75.00}     & \textbf{75.00} &  66.67  &
         \textbf{83.33} & 77.78 &  66.67 &
         \textbf{83.33} & 78.50 &  75.22 \\
         
         Monks-2   & 56.94    & \textbf{60.65} & 60.26 &
         \textbf{63.89} & 63.66 & 61.13 & 
         61.31 & \textbf{65.15} & 63.49  \\
         
         Monks-3   & \textbf{97.22}     & \textbf{97.22} &  \textbf{97.22} &
         94.44 & \textbf{97.22}  & \textbf{97.22} &
         95.37 & \textbf{97.22} & 94.59 \\
         
         Tic-tac-toe  & 68.23 & \textbf{68.76} & 68.31 & 
         72.40 & 70.05 & \textbf{75.65} & 
         \textbf{81.77} & 75.27 & 80.16 \\
        \bottomrule
    \end{tabular}
    \begin{tabular}{l|rrr|rrr}
       \addlinespace[1.5em]
       \toprule
         & \multicolumn{6}{c}{Test accuracy (\%)} \\
         \cmidrule{2-7}
        Dataset & \multicolumn{3}{c}{$k=5$} & \multicolumn{3}{c}{$k=6$}  \\
        \cmidrule{2-7}
        &   \hspace{0.35cm}DT   &  \hspace{0.35cm} DNF &  \hspace{0.35cm}  \(\overline{\text{DNF}}\) &   \hspace{0.35cm} DT   & \hspace{0.35cm} DNF & \hspace{0.35cm}  \(\overline{\text{DNF}}\)\\ 
        \midrule

         Balance-scale &  92.96 & 92.05 & \textbf{93.28} &
         92.18 & 90.58 & \textbf{93.10} \\
        
         Banknote    &  \textbf{98.25} & 90.25 & 88.52
         & \textbf{99.02} & 90.01 & 88.52 \\
         
         Car-evaluation   &    \textbf{92.83} & 82.51 & 91.48 & 
         \textbf{93.64} & 82.97 & 91.79 \\
         
         Compas discretized   &  \textbf{67.31} &  66.40 & \textbf{67.31} &
         66.97 & 66.51 & \textbf{67.70}  \\
         
         Indians Diabetes   &  77.64  & \textbf{79.66} &  77.23 &
         77.43 &  \textbf{79.57} & 76.97 \\
         
         Iris   &    \textbf{98.00} & \textbf{98.00} & \textbf{98.00} &
         \textbf{ 98.00 } & 97.27 & \textbf{98.00} \\
         
         Lymph   &  85.00 & 81.93 & \textbf{85.93} &
         84.27 & 80.40 & \textbf{86.27}\\
         
         Monks-1   & \textbf{83.33} & 82.20 & 77.41 & 
         83,33 &  \textbf{91.17}  & 80.52\\
         
         Monks-2   & 68.26 & 67.32 &  \textbf{68.33} &
         \textbf{78.85} & 67.55 & 73.63 \\
         
         Monks-3   & 89.81 & 89.00 &  \textbf{92.46} &
         \textbf{92.59} & 87.09 &  88.19 \\
         
         Tic-tac-toe  &  \textbf{90.98} & 75.52 & 78.07 &
         \textbf{92.28} & 77.55 & 79.38 \\
         
        \bottomrule
    \end{tabular}
    \caption{Test accuracy of depth-k decision trees and nested k-DNFs}
     \label{tab:acc}
\end{table*}

\subsection{Results}

As a first test, the proposed heuristic successfully found the 2-DNF with 2 terms that perfectly match the full truth-table
generated from $\kappa(a,b,c,d) = (a \land b) \lor (c \land d)$. In contrast, the CART algorithm required a depth of 4 to create a decision tree that fits the data exactly, as mentioned in Example~\ref{ex:2dnf}.

The rest of our experimental assessment was performed on the datasets described above. For a given dataset, 80\% of the dataset was used for training and 20\% for testing, except for the Monks datasets, where the test set is provided separately and consists of 432 examples, consisting of all possible combinations of the feature-values. The average performance across five split experiments is reported. For each of the two training algorithms, the experiment is run 10 times and the average accuracy is computed on the test set. Table~\ref{tab:acc} shows the accuracy of our nested $k$-DNFs (column DNF) and the decision trees generated by CART with a fixed maximum depth of $k$ (column DT).
Given the asymmetry of nested $k$-DNFs with respect to complementation, we repeated the
experiment, learning a nested $k$-DNF model for $\overline{\kappa}$ 
rather than $\kappa$: results are
reported in column $\overline{\text{DNF}}$.

The aim in using different datasets for experimentation is to assess whether the proposed heuristic can actually find a nested $k$-DNF that accurately represents the underlying structure of the data, as decision trees do. The results indicate variability in accuracy across different datasets, with nested k-DNFs outperforming depth-$k$ decision trees in some cases, and vice versa in others. Overall, the results achieved by both depth-k decision trees and nested $k$-DNFs are comparable. Thus, nested $k$-DNFs emerge as a promising alternative to decision trees, with these initial results highlighting the potential of this family of models.

We also compared the size of the DT and nested $k$-DNF models. Table~\ref{tab:SIZE} shows the average number of leaves in the DTs and the average number of terms in the DNFs across five datasets splits, with both training algorithms executed 10 times per split, and the best-performing model selected from these iterations. A nested $k$-DNF is composed of terms associated with a single class, while a DT contains paths leading to both classes. However, despite this difference, the number of terms is generally less than half of the number of leaves, with a significant number of cases exhibiting an even greater disparity. A similar observation was noted for the number of terms in $\overline{\text{DNF}}$. This suggests that the nested $k$-DNFs are simpler than the DTs in terms of size.

\begin{table*} [!t]
    \centering
    \begin{tabular}{l|rr|rr|rr|rr|rr}
        \toprule
         Dataset & \multicolumn{2}{c}{$k=2$} & \multicolumn{2}{c}{$k=3$} & \multicolumn{2}{c}{$k=4$} & \multicolumn{2}{c}{$k=5$} & \multicolumn{2}{c}{$k=6$}\\
        \cmidrule{2-11}
        &   \hspace{0.35cm}DT   &  \hspace{0.35cm} DNF &   \hspace{0.35cm} DT   & \hspace{0.35cm} DNF &   \hspace{0.35cm} DT  & \hspace{0.35cm}  DNF &   \hspace{0.35cm} DT   & \hspace{0.35cm} DNF  &  \hspace{0.35cm} DT   & \hspace{0.35cm}  DNF \\ 
        \midrule
         Balance-scale & 4.0 &2.0 &8.0 &2.8  &14.6 &3.4 &23.2&7.2  &36.4&13.6  \\
         Banknote    & 4.0 &2.0 &8.0 &2.0  &14.0 &2.0 &21.2&2.2  &27.2&2.6  \\
         Car-evaluation   &3.0 &1.0 &4.0&2.6  &6.0&3.2&   9.8&3.8&16.4&4.4 \\
         Compas discretized   &4.0 &1.8& 8.0&2.6&  15.6&4.2&  29.0&4.6  &51.4&6.4   \\
         Indians Diabetes   & 4.0 &2.0 &8.0&3.2& 15.6&4.2   &27.4&5.0 &43.4&5.8 \\
         Iris   & 3.0&1.6   &4.4&2.0   &4.4&2.0  &4.4&2.6  &4.4&3.2  \\
         Lymph   &  4.0&1.8  &8.0&2.2   &13.2&2.4&  16.2&2.2&18.0&3.0\\
         Monks-1   & 3.0&2.0   & 5.0&3.0&   6.0&3.0 &8.0&5.0  &11.0&6.0 \\
         Monks-2   & 4.0&2.0   &8.0&4.0   &15.0&5.6&   25.0&6.6&   40.0&8.6\\
         Monks-3   &  4.0 &1.0   &6.0 &1.0  &9.0 &1.0  &11.0 &3.6  &13.0 &9.0\\
         Tic-tac-toe  & 4.0&1.2  &8.0&4.2&  14.0&8.4&  22.4&10.6&  33.8&15.0 \\
        \bottomrule
    \end{tabular}
    \caption{The number of leaves in the DT and the number of terms in the nested $k$-DNF}
    \label{tab:SIZE}
\end{table*}

\section{Conclusion and future work}

A machine-learning model can be deemed interpretable if each
of its decisions has an explanation that is intelligible by a human user.
We formalized this definition of interpretability based on
abductive or counterfactual explanations of size at most 
a small constant $k$. In the case of binary
classifiers over boolean domains, we showed that this definition
is equivalent to the classifier and its complement both being
expressible as $k$-DNFs. Depth-$k$ decision trees are the most
well-known example of a family of models satisfying this definition.
Decision trees are widely used either directly or as surrogate
models to provide explanations. This paper investigated the
existence of other families of interpretable models.

We introduced a graph-theoretical sufficient condition for interpretability 
in terms of maximum induced matchings of DNF formulas, before giving a novel concrete
family of interpretable models which we call nested $k$-DNFs.
We showed experimentally that a simple heuristic algorithm
produces nested $k$-DNFs whose accuracy is comparable with
depth-$k$ decision trees found by CART. 

An intriguing open question is whether there exist more general families of interpretable DNFs that could achieve better accuracy than decision trees. In contrast to decision trees of depth $k$, the property of a function being expressible as a nested $k$-DNF is not invariant under complementation in general. In addition, nested $k$-DNFs cannot contain more than $k^2$ distinct literals. These limitations come from our definitions and do not arise from fundamental technical reasons, so we believe there is ample room for further improvement. 

Finally, our observations during the experiments revealed some variability in the test accuracy of the nested $k$-DNFs across different runs.
This observation suggests that significantly better results could be achieved by using more sophisticated heuristics. In particular, it would be interesting to compare optimal nested $k$-DNFs and optimal depth-$k$ decision trees.

\section*{Acknowledgments}
This work was funded by the French National Research Agency 
(ANR) under grant agreement no. ANR-23-CE25-0009.
We would also like to thank Aur\'elie Hurault for many insightful
comments.

\bibliographystyle{plain}
\bibliography{refs}

\begin{thebibliography}{10}

\bibitem{AmgoudB22}
Leila Amgoud and Jonathan Ben{-}Naim.
\newblock Axiomatic foundations of explainability.
\newblock In Luc~De Raedt, editor, {\em {IJCAI}}, pages 636--642. ijcai.org,
  2022.

\bibitem{complexitybook}
Sanjeev Arora and Boaz Barak.
\newblock {\em Computational Complexity: A Modern Approach}.
\newblock Cambridge University Press, USA, 1st edition, 2009.

\bibitem{AudemardBBKLM21}
Gilles Audemard, Steve Bellart, Louenas Bounia, Fr{\'{e}}d{\'{e}}ric Koriche,
  Jean{-}Marie Lagniez, and Pierre Marquis.
\newblock On the computational intelligibility of boolean classifiers.
\newblock In Meghyn Bienvenu, Gerhard Lakemeyer, and Esra Erdem, editors, {\em
  {KR}}, pages 74--86, 2021.

\bibitem{BarceloM0S20}
Pablo Barcel{\'{o}}, Mika{\"{e}}l Monet, Jorge P{\'{e}}rez, and Bernardo
  Subercaseaux.
\newblock Model interpretability through the lens of computational complexity.
\newblock In Hugo Larochelle, Marc'Aurelio Ranzato, Raia Hadsell,
  Maria{-}Florina Balcan, and Hsuan{-}Tien Lin, editors, {\em NeurIPS}, 2020.

\bibitem{bergehypergraphs}
Claude Berge.
\newblock Hypergraphs: Combinatorics of finite sets.
\newblock 1989.

\bibitem{certif2}
Guy Blanc, Caleb Koch, Jane Lange, and Li{-}Yang Tan.
\newblock The query complexity of certification.
\newblock In Stefano Leonardi and Anupam Gupta, editors, {\em {STOC} '22: 54th
  Annual {ACM} {SIGACT} Symposium on Theory of Computing}, pages 623--636.
  {ACM}, 2022.

\bibitem{DBLP:books/wa/BreimanFOS84}
Leo Breiman, J.~H. Friedman, Richard~A. Olshen, and C.~J. Stone.
\newblock {\em Classification and Regression Trees}.
\newblock Wadsworth, 1984.

\bibitem{kr23}
Cl{\'{e}}ment Carbonnel, Martin~C. Cooper, and Jo{\~{a}}o Marques{-}Silva.
\newblock Tractable explaining of multivariate decision trees.
\newblock In Pierre Marquis, Tran~Cao Son, and Gabriele Kern{-}Isberner,
  editors, {\em {KR}}, pages 127--135, 2023.

\bibitem{certif1}
Siddhesh Chaubal and Anna G{\'{a}}l.
\newblock Diameter versus certificate complexity of boolean functions.
\newblock In Filippo Bonchi and Simon~J. Puglisi, editors, {\em 46th
  International Symposium on Mathematical Foundations of Computer Science,
  {MFCS}}, volume 202 of {\em LIPIcs}, pages 31:1--31:22. Schloss Dagstuhl -
  Leibniz-Zentrum f{\"{u}}r Informatik, 2021.

\bibitem{ecai/CooperA23}
Martin~C. Cooper and Leila Amgoud.
\newblock Abductive explanations of classifiers under constraints: Complexity
  and properties.
\newblock In Kobi Gal, Ann Now{\'{e}}, Grzegorz~J. Nalepa, Roy Fairstein, and
  Roxana Radulescu, editors, {\em {ECAI}}, volume 372 of {\em Frontiers in
  Artificial Intelligence and Applications}, pages 469--476. {IOS} Press, 2023.

\bibitem{AIJ23}
Martin~C. Cooper and Jo{\~{a}}o Marques-Silva.
\newblock Tractability of explaining classifier decisions.
\newblock {\em Artif. Intell.}, 316, 2023.

\bibitem{blossom}
Emir Demirovic, Emmanuel Hebrard, and Louis Jean.
\newblock Blossom: an anytime algorithm for computing optimal decision trees.
\newblock In Andreas Krause, Emma Brunskill, Kyunghyun Cho, Barbara Engelhardt,
  Sivan Sabato, and Jonathan Scarlett, editors, {\em {ICML}}, volume 202, pages
  7533--7562. {PMLR}, 2023.

\bibitem{DurdymyradovM24}
Kerven Durdymyradov and Mikhail Moshkov.
\newblock Bounds on depth of decision trees derived from decision rule systems
  with discrete attributes.
\newblock {\em Ann. Math. Artif. Intell.}, 92(3):703--732, 2024.

\bibitem{HuangIICA022}
Xuanxiang Huang, Yacine Izza, Alexey Ignatiev, Martin~C. Cooper, Nicholas
  Asher, and Jo{\~{a}}o Marques{-}Silva.
\newblock Tractable explanations for d-{DNNF} classifiers.
\newblock In {\em {AAAI}}, pages 5719--5728. {AAAI} Press, 2022.

\bibitem{aaai/IgnatievIS022}
Alexey Ignatiev, Yacine Izza, Peter~J. Stuckey, and Jo{\~{a}}o Marques{-}Silva.
\newblock Using {MaxSAT} for efficient explanations of tree ensembles.
\newblock In {\em {AAAI}}, pages 3776--3785. {AAAI} Press, 2022.

\bibitem{sat/IgnatievS21}
Alexey Ignatiev and Jo{\~{a}}o Marques{-}Silva.
\newblock {SAT}-based rigorous explanations for decision lists.
\newblock In Chu{-}Min Li and Felip Many{\`{a}}, editors, {\em Theory and
  Applications of Satisfiability Testing - {SAT}}, volume 12831 of {\em Lecture
  Notes in Computer Science}, pages 251--269. Springer, 2021.

\bibitem{IgnatievNM19}
Alexey Ignatiev, Nina Narodytska, and Jo{\~{a}}o Marques{-}Silva.
\newblock Abduction-based explanations for machine learning models.
\newblock In {\em {AAAI}}, pages 1511--1519. {AAAI} Press, 2019.

\bibitem{jair/IzzaIM22}
Yacine Izza, Alexey Ignatiev, and Jo{\~{a}}o Marques{-}Silva.
\newblock On tackling explanation redundancy in decision trees.
\newblock {\em J. Artif. Intell. Res.}, 75:261--321, 2022.

\bibitem{ijcai/Izza021}
Yacine Izza and Jo{\~{a}}o Marques{-}Silva.
\newblock On explaining random forests with {SAT}.
\newblock In Zhi{-}Hua Zhou, editor, {\em {IJCAI}}, pages 2584--2591.
  ijcai.org, 2021.

\bibitem{Joao24}
Jo{\~{a}}o Marques{-}Silva.
\newblock Logic-based explainability: Past, present {\&} future.
\newblock {\em CoRR}, abs/2406.11873, 2024.

\bibitem{neurips20}
Jo{\~{a}}o Marques{-}Silva, Thomas Gerspacher, Martin~C. Cooper, Alexey
  Ignatiev, and Nina Narodytska.
\newblock Explaining naive {B}ayes and other linear classifiers with polynomial
  time and delay.
\newblock In Hugo Larochelle, Marc'Aurelio Ranzato, Raia Hadsell,
  Maria{-}Florina Balcan, and Hsuan{-}Tien Lin, editors, {\em {NeurIPS}}, 2020.

\bibitem{icml21}
Jo{\~{a}}o Marques{-}Silva, Thomas Gerspacher, Martin~C. Cooper, Alexey
  Ignatiev, and Nina Narodytska.
\newblock Explanations for monotonic classifiers.
\newblock In Marina Meila and Tong Zhang, editors, {\em {ICML}}, volume 139,
  pages 7469--7479. {PMLR}, 2021.

\bibitem{Molnar}
Christoph Molnar, Giuseppe Casalicchio, and Bernd Bischl.
\newblock Interpretable machine learning - {A} brief history, state-of-the-art
  and challenges.
\newblock {\em CoRR}, abs/2010.09337, 2020.

\bibitem{Rudin19}
Cynthia Rudin.
\newblock Stop explaining black box machine learning models for high stakes
  decisions and use interpretable models instead.
\newblock {\em Nat. Mach. Intell.}, 1(5):206--215, 2019.

\bibitem{ShihCD18}
Andy Shih, Arthur Choi, and Adnan Darwiche.
\newblock A symbolic approach to explaining bayesian network classifiers.
\newblock In J{\'{e}}r{\^{o}}me Lang, editor, {\em {IJCAI}}, pages 5103--5111.
  ijcai.org, 2018.

\end{thebibliography}

\end{document}